\documentclass[12pt]{article}
\usepackage{latexsym,amsmath,amssymb,graphicx}

\newtheorem{thm}{Theorem}[section]
\newtheorem{lemma}[thm]{Lemma}

\newtheorem{fig}{Figure}

\newenvironment{proof}{\begin{trivlist}\item[]{\em Proof.\/\ }}%
                      {\hfill$\Box$ \\ \end{trivlist}}

\title{On Ullman's theorem in computer vision}
\author{Oliver Knill and Jose Ramirez-Herran
\footnote{Harvard University, this research was supported by the Harvard Extension School}}
\date{August 17, 2007}

\begin{document}
\maketitle

\abstract{
Both in the plane and in space, we invert the nonlinear Ullman transformation
for 3 points and 3 orthographic cameras. While Ullman's theorem assures a unique reconstruction 
modulo a reflection for 3 cameras and 4 points, we find a locally unique reconstruction
for 3 cameras and 3 points. Explicit reconstruction formulas allow to 
decide whether picture data of three cameras seeing three points can be realized 
as a point-camera configuration.
}

\section{Introduction}

%  the problem in general 
Ullman's theorem in {\bf computer vision} is a prototype of a {\bf structure from motion result}.
Given $m$ planes in space and $n$ points for which we know the orthogonal projections of the points on the planes,
we want to recover the planes and the points. 
The problem can also be formulated as follows: given a fixed orthographic camera, and a point configuration 
which undergoes a rigid transformation. Taking $m$ pictures of this {\bf rigid $n$-body motion}, 
how do we reconstruct the body as well as its motion? Ullman's theorem is often cited as follows:
"For rigid transformations, a unique metrical reconstruction is 
known to be possible from three orthographic views of four points" \cite{Ullman}. \\

While 3 points in general position can be reconstructed from 2 orthographic  
projections, if the image planes are known, one needs 3 views to recover also 
the camera parameters. While Ullman's theorem states {\bf four points}, 
{\bf three points are enough} for a {\bf locally unique} reconstruction. 
Actually, already Ullman's proof demonstrated this. We produce algebraic inversion formulas in this paper. 
Ullman's transformation is a nonlinear polynomial map which computer algebra systems is unable to invert. 
Ullman's proof idea is to reconstruct the intersection lines of the planes first, computer algebra
systems produce complicated solution formulas because quartic polynomial equations 
have to be solved. Fortunately, it is possible to reduce the complexity. \\

The fact that four points produce an overdetermined system has been described by Ullman as follows:
"the probability that three views of four points not moving rigidly together will admit a 
rigid interpretation is low. In fact, the probability is zero." (\cite{Ullman} Ullman section 4.4 p. 149.)
In other words, for a given 3 sets of 4 points in the plane, a reconstruction 
of a 3D scene with 4 points and 3 cameras is in general not possible. Indeed, the
camera-point space has a much lower dimension than the space of photo configurations. 
Assuming 3 cameras, it is for 3 points, that the number of unknowns matches the number of equations.
Let's look at a simple dimensional analysis for three points and three cameras.
See \cite{KnillRamirezInequality} for arbitrary cameras.
One point can be fixed at the orign and one camera can be placed as the $xy$-plane. 
Because a general rotation in space needs 3 parameters, there are 3 unknowns for each camera.
Because the projections onto the first camera already give the $x,y$ coordinates of the planes, 
we only need to know the heights of the points. This leaves us with $8$ unknowns: 
$C = (z_1,z_2,\theta_1,\phi_1,\gamma_1, \theta_2, \phi_2,\gamma_2)$. Every point-camera configuration 
produces 2 real coordinates $Q_j(P_i)$ so that two cameras provide us with a total $8$ of data points. Let
$R_j$ denote the rotation matrices defined by the angles $\theta_j,\phi_j,\gamma_j$, let
$p_j=R_j (1,0,0),q_j=R_j (0,1,0)$ be the basis in the camera planes and denote by $P_i=(x_i,y_i,z_i)$ the points.
The nonlinear {\bf structure from motion transformation} $F$ from $R^8$ to $R^8$ is
$$ F(C) = (p_2 \cdot P_2, q_2 \cdot P_2, p_2 \cdot P_3, q_2 \cdot P_3,p_3 \cdot P_2, q_3 \cdot P_2, p_3 \cdot P_3, q_3 \cdot P_3)  \; . $$
It needs to be inverted explicitely on the image $\{ F(C) \; | \; {\rm det}(F)(C) \neq 0 \; \}$.
Because $F$ is not surjective, it is an interesting question to characterize the image data which allow 
a reconstructions. By the implicity function theorem, the boundary of this set is 
$\{ F(C) \; | \; {\rm det}(F)(C)=0 \; \}$.

\begin{center}
\scalebox{0.40}{\includegraphics{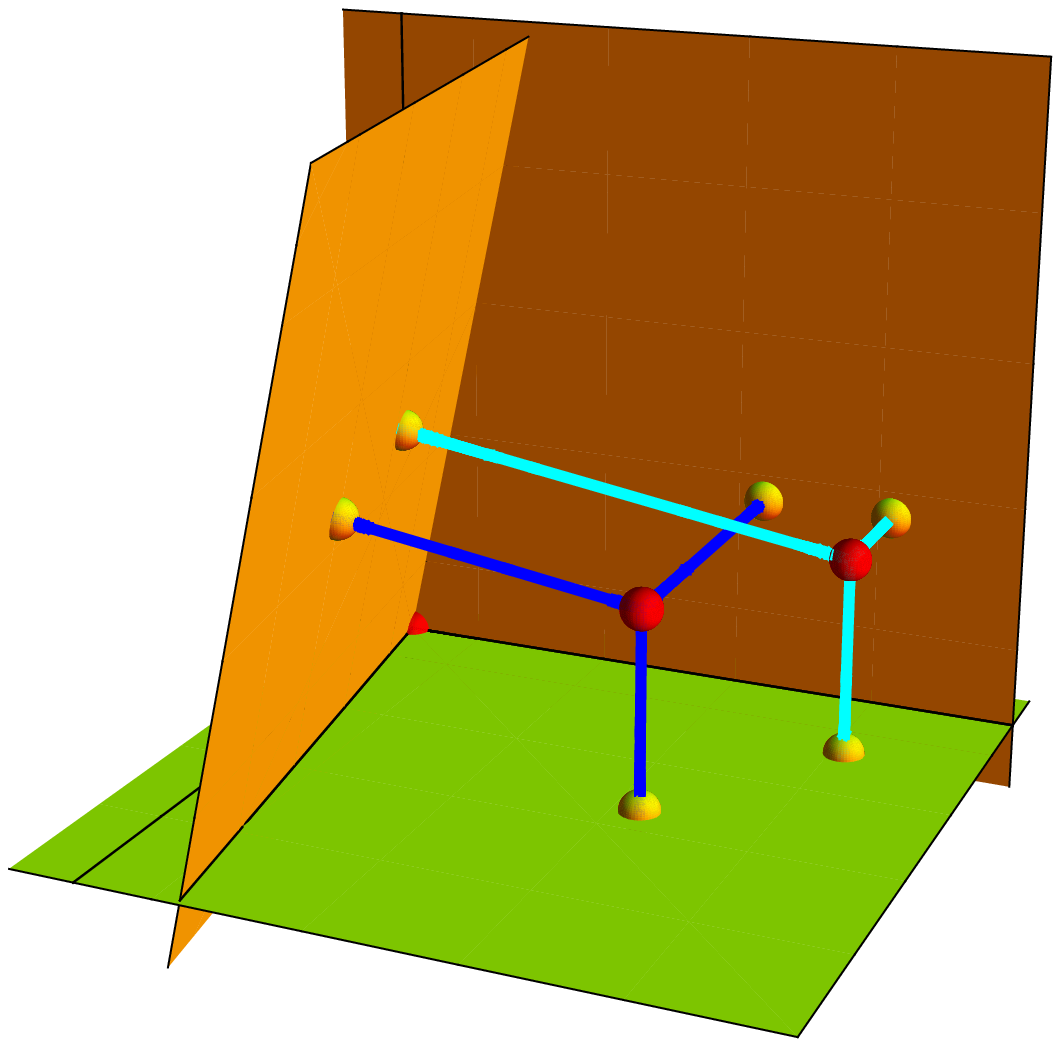}}
$\rightarrow$
\scalebox{0.80}{\includegraphics{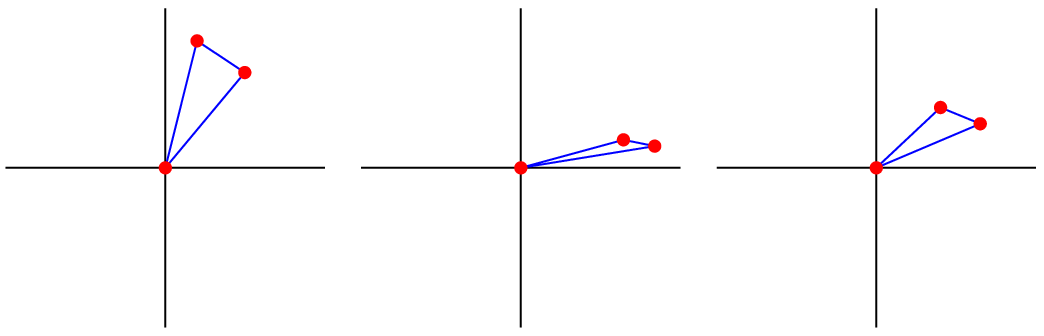}}
\end{center}
\begin{fig}
Given three planes and three points, the map $F$ produces projections of the points
on the planes. The problem is to rebuild the planes and the points from these
data. The nonlinear map $F$ from camera point configurations to the photographic data space 
is finite to one. The set for which $\{ {\rm det}(F)=0 \}$ is mapped to the boundary of the image 
which is a proper subset of all possible photographic data. 
\end{fig}

As Ullman has pointed out, the reconstruction is not unique: changing the
signs of $P_i,p_i$ and $q_i$ does not change the image point. 
This is the case for any number of points $P_i$ and any number of cameras spanned by 
a vector pair $(q_i,p_i)$ because the image data
$(P_i \cdot q_j,P_i \cdot p_j)$ from which we want to recover $P_i,p_j,q_j$ are the same if $P_i,p_j,q_j$
are replaced by $-P_i,-p_j,-q_j$. So, even with arbitraryly many cameras, structure is never uniquely 
recoverable. Less ambiguities occur with a four'th point $P_4$ 
as Ullman has shown. While adding this four'th point reduces the number of mirror 
ambiguities, it adds constraints so that almost all image data are unrealizable.
Indeed, a reconstruction is only possible on a codimension 3 manifold of photographic 
data situations.  \\

Historically, Ullman's theorem is a key result. It provides a link between 
computer vision, psychology, artificial intelligence and geometry. 
We should point out that while Ullman's setup is very familiar to {\bf computer aided design} CAD, 
where an engineer works with three views of the three-dimensional object, there is an essential 
difference with structure from motion: unlike in CAD, we do not know the cameras and 
finding them is part of the reconstruction problem. 

%Ullman first had a theorem for 3 cameras and 5 points in 1977, which was improved to 4 
%non-coplanar points by Fremlin. 

\section{Ullman's theorem in two dimensions}

The two-dimensional Ullman problem is interesting by itself. The algebra is simpler than in three
dimensions but it is still not completely trivial. 
The two dimensional situation plays an important role in the 
3 dimensional problem because the three dimensional situation reduces to it if the three planes 
have coplanar normal vectors. 
Let's first reformulate the two-dimensional Ullman theorem in a similar fashion as Ullman did.
A more detailed reformulation can be found at the end of this section.

\begin{thm}[Two dimensional Ullman theorem]
In the plane, three different orthographic affine camera images of three noncollinear points
determines both the points and cameras in general up to a reflections at the camera planes. 
\end{thm}

\begin{center}
\parbox{6.2cm}{\scalebox{0.50}{\includegraphics{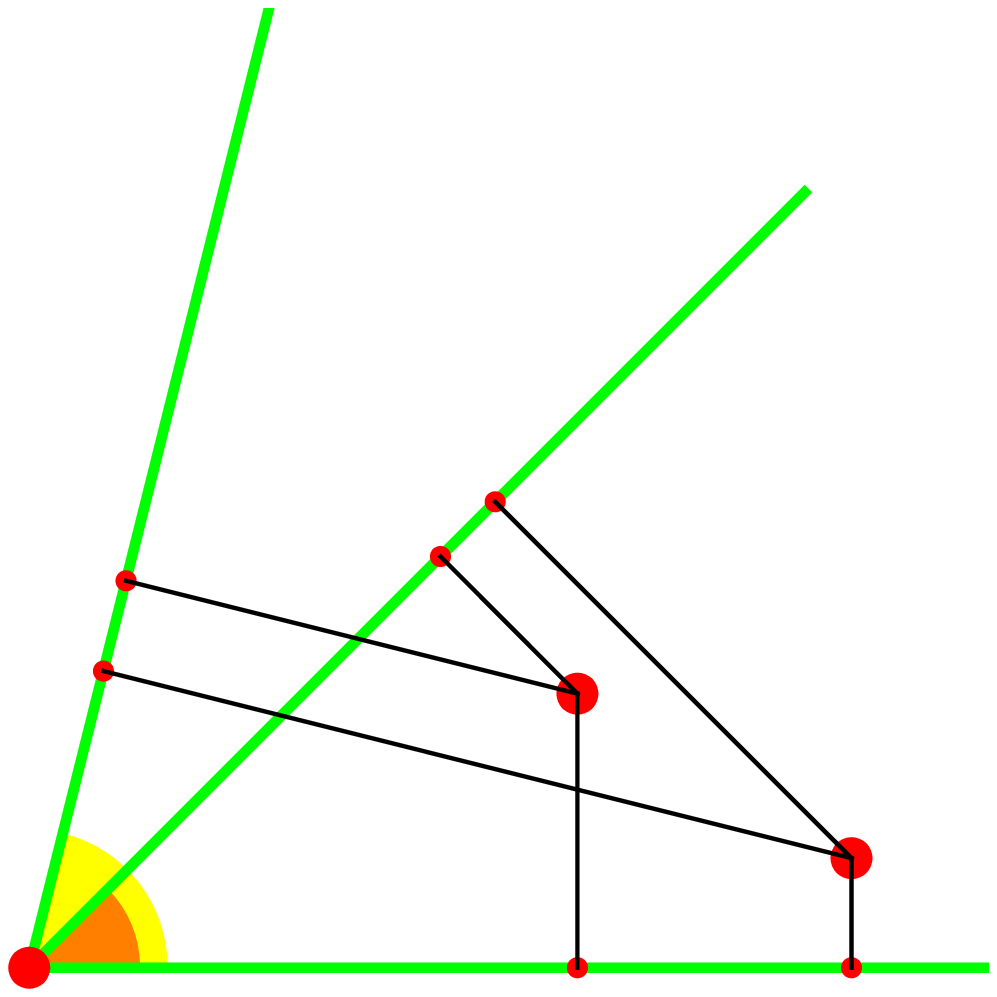}}}
\end{center}
\begin{fig}
The setup for the structure of motion problem with three orthographic cameras and three points
in two dimensions. One point is at the origin, one camera is the $x$-axis. The problem is to find
the $y$ coordinates of the two points as well as the two camera angles from the scalar 
projections onto the lines. 
\end{fig}

\begin{proof}
With the first point $P_1$ at the origin $(0,0)$, the translational symmetry of the problem is fixed.
Because cameras can be translated without changing the pictures, we can assume that all 
camera planes go through the origin $(0,0)$. By having the first camera as the $x$-axis,
the rotational symmetry of the problem is fixed. We are left with 6 unknowns, the
$y$-coordinates of the two points $(x_i,y_i)$ and the directions $v_i = (\cos(\alpha_i),\sin(\alpha_i))$
of the other cameras. Because the $x$-coordinates of the points can directly be seen 
by the first camera, only $4$ unknowns remain: $y_1,y_2,\alpha_1,\alpha_2$. 
We know the image data $a_{ij}$, the scalar component of the projection of the point onto the camera. 
This is what the photographer $j$ sees from the point $P_i$. We have the equations
$$ a_{ij} = P_i \cdot v_j $$
which are nonlinear in the angles $\alpha_i$ for given $x_1,x_2$. The problem is to invert 
the structure from motion map:
$$ F(\left[ \begin{array}{c} y_1 \\ y_2 \\ \alpha_1 \\ \alpha_2 \end{array} \right] )= 
     \left[ \begin{array}{c}      x_1 \cos(\alpha_1) + y_1 \sin(\alpha_1) \\
                                  x_1 \cos(\alpha_2) + y_1 \sin(\alpha_2) \\
                                  x_2 \cos(\alpha_1) + y_2 \sin(\alpha_1) \\
                                  x_2 \cos(\alpha_2) + y_2 \sin(\alpha_2) 
            \end{array} \right]  =
     \left[ \begin{array}{c} a_{11} \\ a_{12} \\ a_{21} \\ a_{22} \end{array} \right] $$
on $R^4$.  It has the Jacobian determinant 
$$ {\rm det}(DF) = \sin(\alpha_1) \sin(\alpha_2) \sin(\alpha_1-\alpha_2) (x_1 y_2-x_2 y_1) \; . $$
We see that it is locally invertible on the image if and only if the three cameras are all different and if the three 
points are not collinear. Because changing the signs of $P_i$ and $v_j$ does not alter the 
values $a_{ij} = P_i \cdot v_j$ seen by the photographer, we have a global reflection 
ambiguity in the problem. We will give explicit solution formulas below. 
\end{proof}

Can this result be improved? What is the image of the map $F$ and
how can one characterize triples of pictures for which a reconstruction is possible?  \\

To see whether the result is sharp, lets look at the dimensions. In dimension $d$, an affine 
orthographic camera is determined
by ${\rm dim}(SO_d) = d (d-1)/2$ parameters and a point by $d$ parameters. We gain $(d-1)$ coordinates for each 
point-camera pair. The global Euclidean symmetry of the problem - rotating and translating the 
point-camera configuration does not change the pictures - gives us the {\bf structure from motion 
inequality for orthographic cameras}
$$  n d + m [ d (d-1)/2 + (d-1) (d-2)/2] \leq (d-1) n m + d + d (d-1)/2 $$
which for $m=3$ and $d=2$ reduces to
$2 n + 6 \leq 3 n  + 2 + 1$ showing that $n=3$ is sharp. 
See \cite{KnillRamirezInequality} for more details. 

\begin{center}
\parbox{7.5cm}{\scalebox{0.45}{\includegraphics{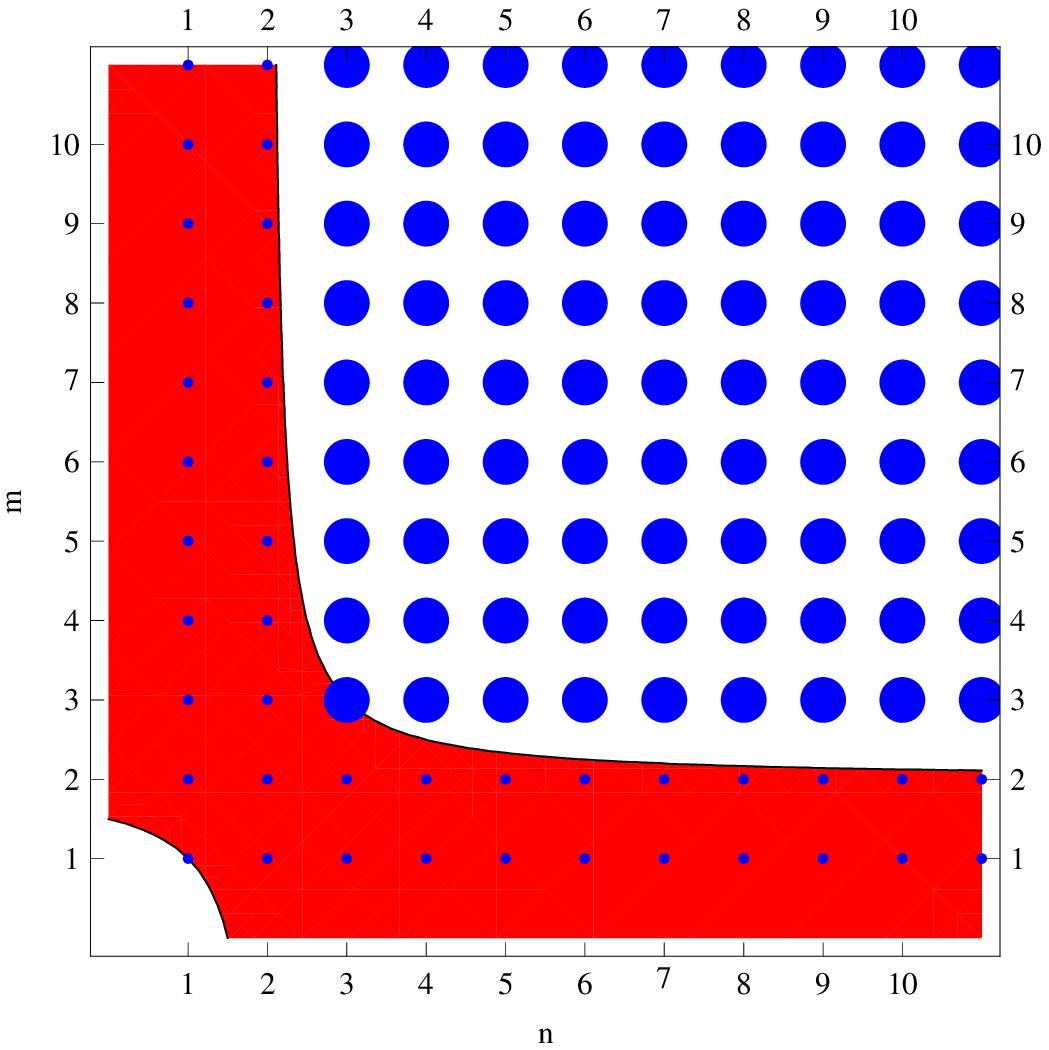}}}
\end{center}
\begin{fig}
The set of dimension pairs $(n,m)$ for which a reconstruction is not possible
by the structure from motion inequality. We see that the point $(n,m) = (3,3)$ is 
the only point, where we have equality. 
\end{fig}

\begin{lemma}
If the points are collinear, a nontrivial deformation of the point-camera configurations is
possible with arbitrarily many points and arbitrarily many cameras.
\end{lemma}
\begin{proof}
The first  point is fixed at the origin $P_1=(0,0)$. Let $P_2(t) = (1,1+t)$. Draw arbitrarily
many cameras for $t=0$. If the point $P_2(t)$ is changed, the cameras can be rotated so that
the scalar projection of $P_2(t)$ onto the lines stay the same. Also the pictures of any 
scalar multiple $P_k(t) = \lambda_k P_2(t)$ stay the same.  
\end{proof}

\begin{center} \parbox{6.5cm}{\scalebox{0.60}{\includegraphics{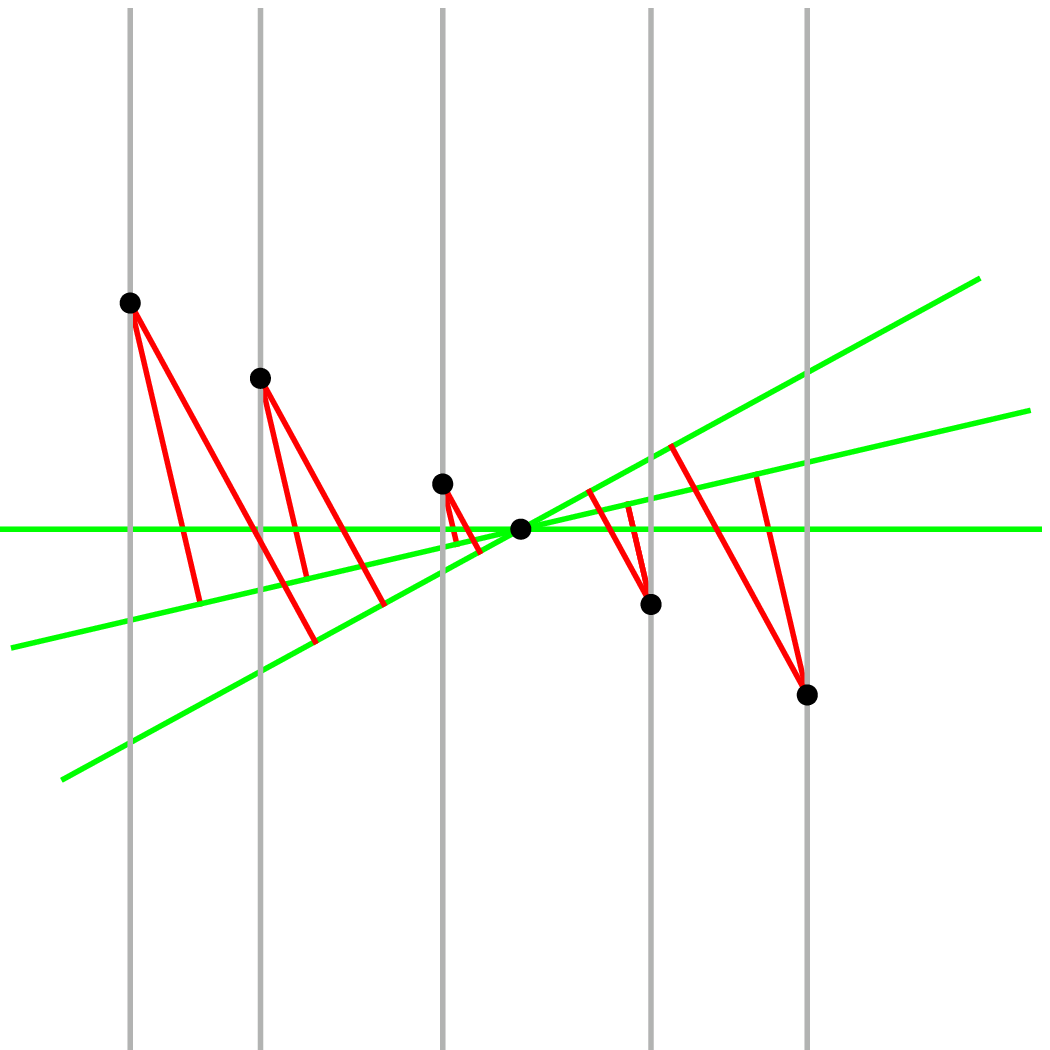}}} \end{center}
\begin{fig}
For arbitrarily many orthographic affine cameras and arbitrarily many points,
there are ambiguities if the points are collinear.
The picture shows 3 cameras and 6 points. The deformation of one camera
plane deformes the points which in turn forces an adjustment of the other camera planes.
\end{fig}

It is also not possible to reduce the number of cameras to two cameras, stereo vision.

\begin{lemma}
With two cameras, a deformation is possible with arbitrarily many points.
\end{lemma}
\begin{proof}
One camera is the x-axes. Take $n$ points. We can move them so that the projection onto the
$x$-axis stays the same and so that the scalar projection to the second camera keeps the 
same distance from the origin.

\begin{center} \parbox{6.5cm}{\scalebox{0.60}{\includegraphics{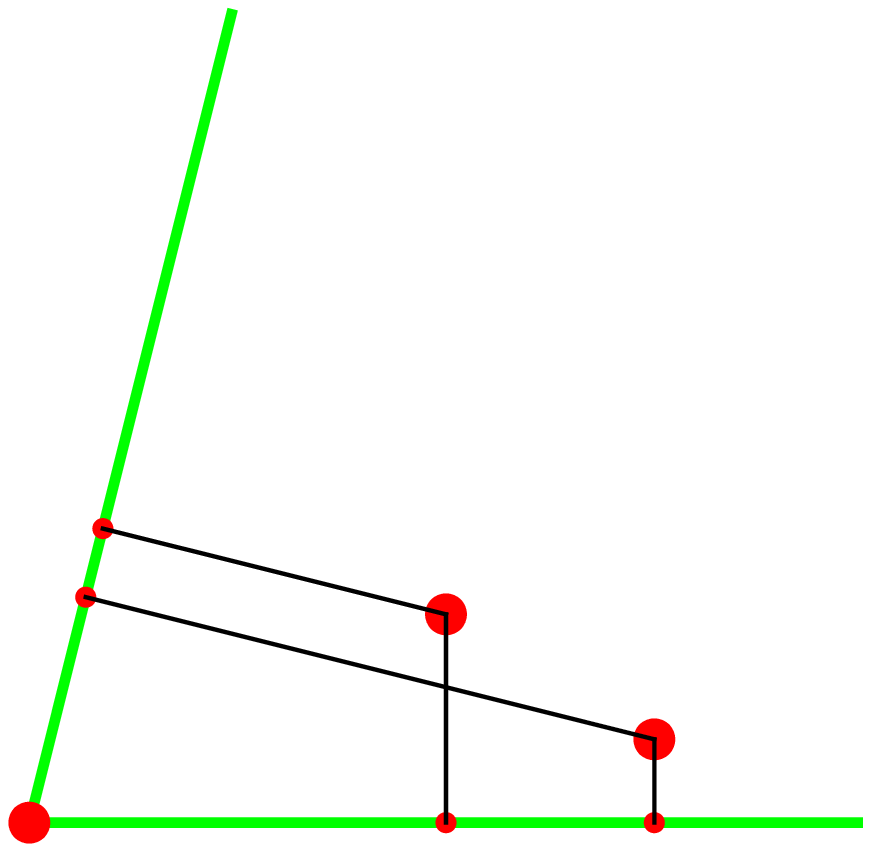}}} \end{center}
\begin{fig}
With two cameras, there are always ambiguities. Fix the first camera as the $x$-axes and
let the second camera be the line $r(s) = s v$. Take image coordinates $x_1,\dots,x_n$ and
$s_1,\dots,s_n$. Define $P_i$ as the intersection to the lines $x=x_i$ with lines perpendicular
to $v$ through $r(s_i)$. Now, if we deform the second camera by turning the vector $v$, we
also have a deformation of the points $P_i$ without changing the image data.
\end{fig}

Is the nonlinear map $F$ surjective? The answer is not obvious because intuition does 
not help much. We have difficulties to visualize 3 points in the plane and three lines, if we 
know the scalar projections onto the three lines. We tried and failed first to prove that 
the map $F$ is surjective. Indeed, the answer is: no, the map $F$ is not surjective. 
The transformation $F$ maps ${\rm R}^4$ to a proper subset of ${\rm R}^4$. 
The boundary is the image of the set ${\rm det}(D F)=0$ which is the set of situations, where 
points are collinear or a camera collision happens, or a point collision happens. Below we 
will look at explicit configurations which are not in the image. \\

How do we invert the structure from motion map $F$? Let us reformulate the problem in index-free
notation. Assume the point coordinates are $(u,p),(v,q)$ and the camera angles are $\alpha,\beta$.
The picture of the first point has coordinates $(a,b)$ and the picture of the second point has 
coordinates $(c,d)$. We want to find $\alpha,\beta,p,q$ from the equations
\begin{eqnarray*}
    \label{ullman2d}
    u \cos(\alpha) +  p \sin(\alpha) &=& a \\
    u \cos(\beta)  +  p \sin(\beta)  &=& c \\
    v \cos(\alpha) +  q \sin(\alpha) &=& b \\
    v \cos(\beta)  +  q \sin(\beta)  &=& d \; . 
\end{eqnarray*}
(The scalars $p,q$ have no relation with the vectors $p$ and $q$ used in the three
dimensional problem mentioned in the introduction).
After eliminating $p$ and $q$, we get $\alpha,\beta$ as solutions of the two equations
\begin{eqnarray*}
 \frac{\sin(\alpha)}{a-u \cos(\alpha)} &=& \frac{\sin(\beta)}{c-u \cos(\beta)}  \\
 \frac{\sin(\alpha)}{b-v \cos(\alpha)} &=& \frac{\sin(\beta)}{d-v \cos(\beta)} \; . 
\end{eqnarray*}
Solutions as intersections of level curves of two functions on the two dimensional torus. 
With $x=\cos(\alpha),y = \cos(\beta)$ we get
\begin{eqnarray*}
  (1-x^2) (a-u y) &=& (1-y^2) (c-u x) \\
  (1-x^2) (b-v y) &=& (1-y^2) (d-v x) \; . 
\end{eqnarray*}
This system of quadratic equations has explicit solutions. They are the first two
of the following set of 4 solution formulas. The $p$ and $q$ can be obtained directly.
\begin{eqnarray*}
x=\cos(\alpha)&=&\frac{(-c^2+d^2)u^2+2acuv-2bd( ac+uv)+a^2(d^2-v^2)+b^2(c^2+v^2)}{2(bc-ad)(-du+bv)} \\
y=\cos(\beta) &=&\frac{( c^2-d^2)u^2-2acuv+2bd(-ac+uv)+b^2(c^2-v^2)+a^2(d^2+v^2)}{2(bc-ad)( cu-av)} \\
p  &=& (a-u \cos(\alpha))/\sin(\alpha) \\
q  &=& (b-v \cos(\alpha))/\sin(\alpha)  \; . 
\end{eqnarray*}
We see that there are $0,1$ or $2$ real solutions and that the only ambiguities are
$(\alpha,\beta,p,q) \to (-\alpha,-\beta,-p,-q)$ because the original equations require 
$\alpha$ and $\beta$ to switch signs simultaneously. 
\end{proof}

%  Solve[ { (1-x^2) (a-u y) == (1-y^2) (c-u x), (1-x^2) (b-v y) == (1-y^2) (d-v x)},{x,y} ]

\begin{center}
\parbox{7.5cm}{\scalebox{0.85}{\includegraphics{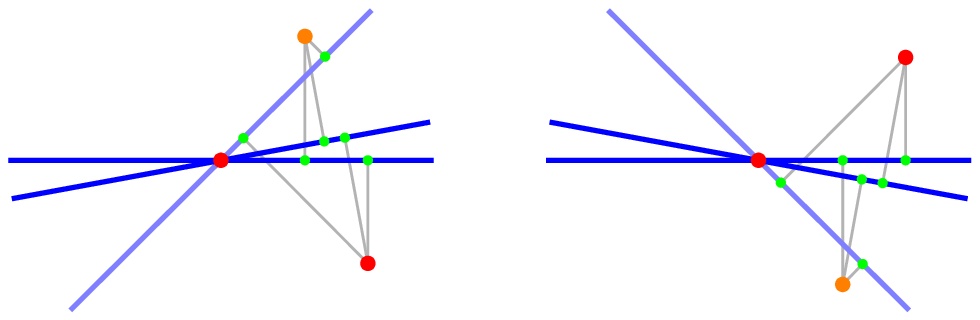}}}
\end{center}
\begin{fig}
The nonlinear map $F$ is $2:1$. Here are the two solutions to a 
typical image set. 
\end{fig}

\section{The image of $F$}

\begin{lemma}
The image of $F$ is a proper subset of all photographic image data. 
\end{lemma}
\begin{proof}
To see this,
lets look at the two-dimensional surface defined by $a=0,d=0,b=u$. The inverse of $F$ 
on this set is 
\begin{eqnarray*}
    \alpha &=& {\rm arccos}(D) \\
    \beta  &=& {\rm arccos}(1-2D^2) \\
    p      &=& -b D (1-D^2)^{-1/2} \\
    q      &=& (2c^2-v^2) (1-D^2)^{-1/2} \; , 
\end{eqnarray*}
where $D=v/2c$. The inverse exists for $|D| \leq 1$. For $|D|>1$
the camera and point data become complex. \\

The image data $(u,v,a,b,c,d)=(1,5,0,1,1,0)$ for example do not correspond to
any actual situation of points $(p,u),(q,v)$ and cameras with direction $\alpha,\beta$. 
It is the situation, where the points are seen with the 
first camera at $1$ and $5$, by the second camera at $0$ and $1$ and by the 
third camera at $1$ and $0$. 
\end{proof}

Let's look a bit closer at the system of nonlinear equations~(\ref{ullman2d}).
Dividing the first to the second shows that $x$ and $y$ are related by a 
M\"obius transform $x=Ay$. A second equation relating $\alpha$ and $\beta$ is obtained from
\begin{eqnarray*}
    \sin(\alpha) (c-u \cos(\beta)) &=& \sin(\beta) ( a-u \cos(\alpha))  \\
    \sin(\alpha) (d-v \cos(\beta)) &=& \sin(\beta) ( b-v \cos(\alpha))
\end{eqnarray*} 
so that
$$  \frac{\sin(\alpha) c - \sin(\beta)a}{\sin(\alpha) d - \sin(\beta) b} = \frac{u}{v} $$
We see that $u/v$ is the M\"obius transform of $\sin(\beta)/\sin(\alpha)$. Because the inverse
of a M\"obius transform is again a M\"obius transform, we know that $\sin(\beta)/\sin(\alpha)$
is the M\"obius transform of $u/v$. This is a real number $G$. The upshot is that we have two equations
\begin{eqnarray*}
 \sin(\beta)  &=& G \sin(\alpha)  \\
 \cos(\beta)  &=& \frac{A+B \cos(\alpha)}{C+D \cos(\alpha)}  
\end{eqnarray*}
for the unknowns $\alpha,\beta$. This quartic equation for $\cos(\beta)$ has $0$ or $2$ solutions
as the explicit formulas show. There is no solution for example, if the M\"obius transformation maps 
$[-1,1]$ to an interval disjoint from $[-1,1]$. \\

{\bf Remark.} Equations~(\ref{ullman2d}) can also be rewritten with 
$r=u+ip, s=v+iq, z=\cos(\alpha) - i \sin(\alpha), 
w=\cos(\beta) - i \sin(\beta)$ and complex $a,b,c,d$ as
$$    r z = a, r w = b, s z=c, s w=d    $$
where $|z|=1,|w|=1,{\rm Re}(r),{\rm Re}(s),{\rm Re}(a),{\rm Re}(b),{\rm Re}(c),{\rm Re}(d)$ 
are known. These are 8 equations for $8$ unknowns. It is a reformulation with
more variables but the equations look simpler.  \\

We summarize:

\begin{thm}
The structure from motion map $F$ for three points and three cameras in the plane 
maps $R^4$ onto a proper subset of $R^4$. The set ${\rm det}(F)=0$ is the
set of camera-point configurations for which two cameras are the same or where the three points are collinear.
This set is maped into the boundary of the data set $F({R}^4)$. The map $F$ is $2:1$ in general. The
only ambiguity is a reflection. The explicit reconstruction of the points $P_2=(u,p),P_3=(v,q)$ and
camera angles $\alpha,\beta$ is 
\begin{eqnarray*}
\cos(\alpha)&=&\frac{(-c^2+d^2)u^2+2acuv-2bd( ac+uv)+a^2(d^2-v^2)+b^2(c^2+v^2)}{2(bc-ad)(-du+bv)} \\
\cos(\beta) &=&\frac{( c^2-d^2)u^2-2acuv+2bd(-ac+uv)+b^2(c^2-v^2)+a^2(d^2+v^2)}{2(bc-ad)( cu-av)} \\
p  &=& (a-u \cos(\alpha))/\sin(\alpha) \\
q  &=& (b-v \cos(\alpha))/\sin(\alpha)  \; . 
\end{eqnarray*}
if the second point $P_2$ is captured by the second camera $Q_2$ at $a$ and by the third camera $Q_3$ at $b$
and the third point $P_3$ is seen with the second camera at $c$ and with the third camera $Q_3$ at $d$.
\end{thm}

\section{Ullman's theorem in three dimensions}

In three dimensions, the structure from motion map $F(P,Q) = Q_j(P_i)$ for $n$ points and $m$ cameras 
is a nonlinear map from $R^{3n} \times SO_3^m \to  R^{2 n m}$. For $m=3,n=3$, this is a map
from $R^{18}$ to $R^{18}$. By fixing the position of the first point and setting
the first camera as the $xy$-plane, we have a map from $R^{6} \times SO_3^2 \to  R^{12}$
Because the first camera freezes the $x$-coordinates of all the points, four of these equations are 
trivial and we have a nonlinear map from $R^{2} \times SO_3^2 \to  R^{8}$. This map from an 8-dimensional
manifold to an 8-dimensional manifold is finite to 1 and difficult to invert directly. 
Computer algebra systems seem unable to do the inversion, even when replacing rotation angles by quaternions, 
which produce polynomial maps. The crucial idea of Ullman is to perform the reconstruction in two steps.

\begin{center}
\parbox{7.5cm}{\scalebox{0.45}{\includegraphics{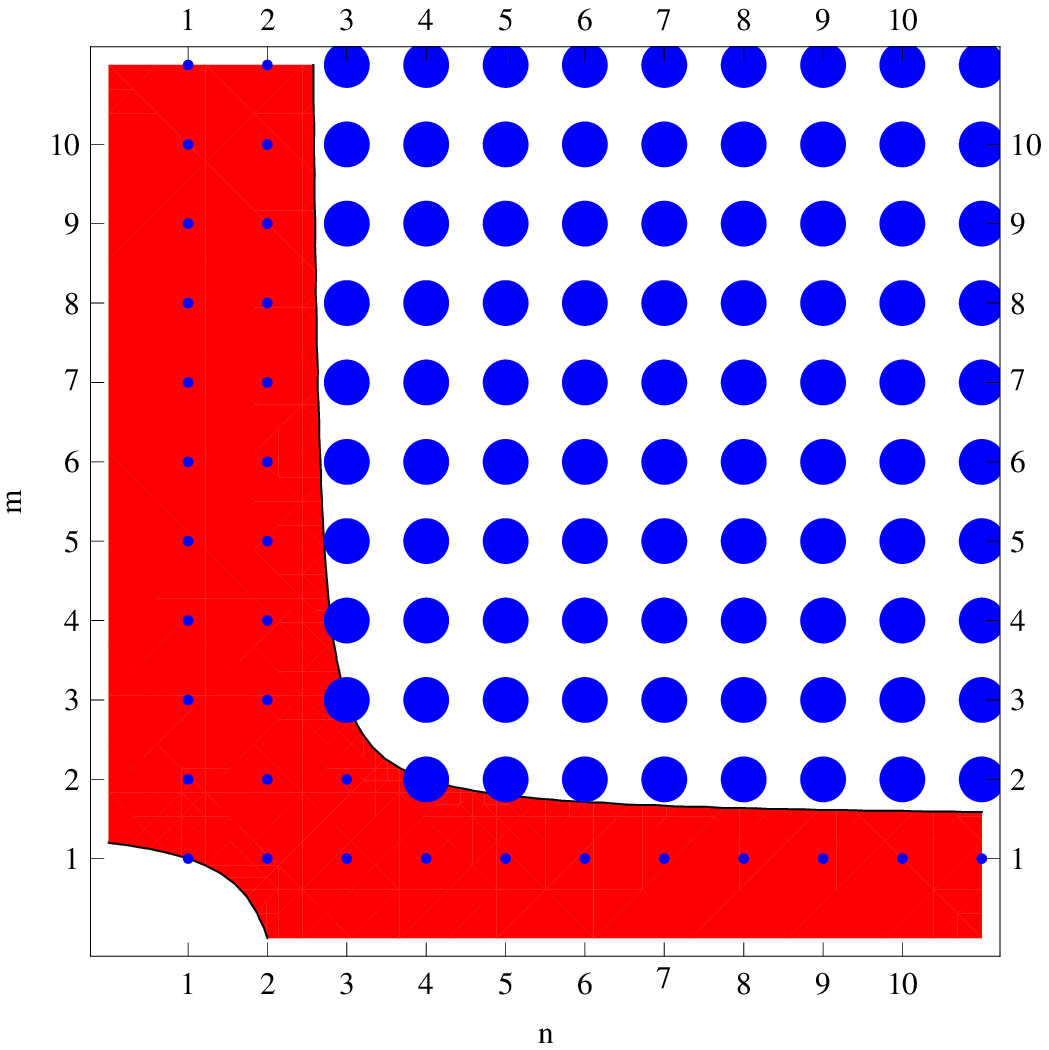}}}
\end{center}
\begin{fig}
The set of dimension pairs $(n,m)$ in three dimensions, for which a
reconstruction is not possible by simple dimensional analysis.
\end{fig}

\begin{thm}[Ullman theorem in three dimensions for 3 points]
In three dimensions, three orthographic pictures of three noncollinear points 
determine both the points and camera positions up to finitely many reflections.
The correspondence is locally unique. 
\end{thm}

We assume that the three planes are different and that the three points are different.
Otherwise, we had a situation with $m<3$ and $n<3$, where finding the inverse is not possible. 
If the normals to the planes are coplanar, that is when the three planes go through 
a common line after some translation, then the problem can be reduced to the 
two-dimensional Ullman problem. 

\begin{center}
\parbox{6.2cm}{\scalebox{0.50}{\includegraphics{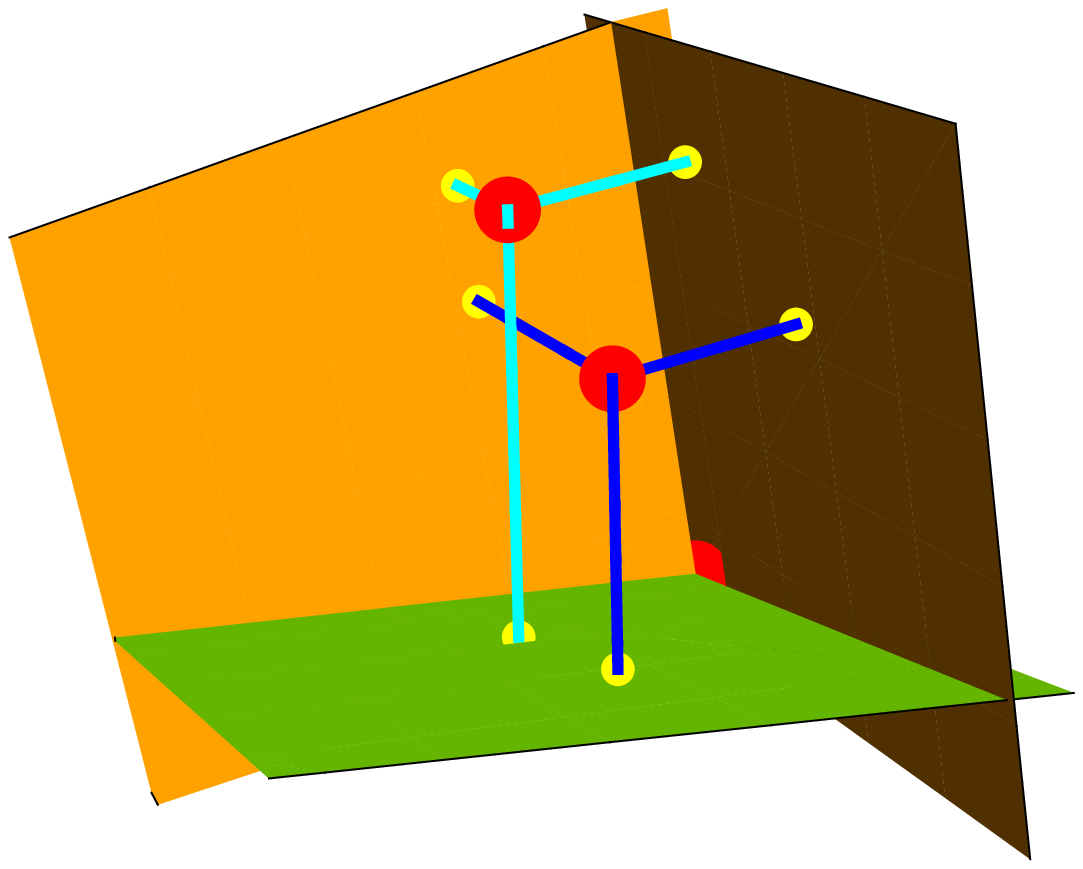}}}
\end{center}
\begin{fig}
The setup for the structure of motion problem with three orthographic cameras and three points
in three dimensions. One point is at the origin, one camera is the $xy$-plane. The problem is to find
the $z$-coordinates of the two points as well as the three Euler angles for each cameras from the 
projections onto the planes. 
\end{fig}

Because Ullman stated his theorem with 4 points and this result is cited so widely
\cite{HoffmanBennett1985,HoffmanBennett1985a,HuAhuja1993,
Bennetetall1993,Koenderink,HoffmanBennett1994,HuAhuja1995,pritt},
we give more details to the proof of Ullman for 3 points. The only 
reason to add a 4'th point is to reduce the number of ambiguities from typically $64$ to 
$2$. We will give explicit solution formulas which provide an explicit reconstruction with 
in the case of $3$ points. One could write down explicit algebraic expressions for the
inverse. 

\begin{proof}
Again we chose a coordinate system so that one of the cameras is the
$xy$-plane with the standard basis $q_0,p_0$. One of the three points 
$P_1=O$ is fixed at the origin.
The problem is to find two orthonormal frames $p_j,q_j$ in space spanning two planes
$S_1$ and $S_2$ through the origin and two points $P_2,P_3$
from the projection data
\begin{equation}
\label{ullman3d} 
a_{ij} = P_i \cdot q_j, b_{ij} = P_i \cdot p_j  \; . 
\end{equation}
The camera $j$ sees the point $P_i$ at the
position $(a_{ij},b_{ij})$. Because an orthonormal 2 frame needs 3 parameters
$(\theta_i,\phi_i,\gamma_i)$ and each point in space has $3$ coordinates, there are
$2 \cdot 3 + 2 \cdot 3=12$ unknowns and 12 equations $a_{ij} = P_i \cdot q_j$ and $b_{ij} = P_i \cdot p_j$,
$i=1,2, j=0,1,2$. Because the projection to the $xy$ plane is known, there
are 4 variables, which can directly be read off. We are left with a nonlinear system of $8$ equations and
$8$ unknowns $(z_1,z_2,\theta_1,\phi_1,\gamma_1,\theta_1,\phi_2,\gamma_2)$. Just plug in
$$p_j = \left[ \begin{array}{c} \cos(\gamma_j) \cos(\theta_j)-\cos(\phi_j) \sin( \gamma_j) \sin(\theta_j) \\
                                -\cos(\phi_j) \cos(\theta_j) \sin(\gamma_j) - \cos(\gamma_j) \sin( \theta_j) \\
                                \sin(\gamma_j)\sin(\phi_j)
               \end{array} \right] $$
$$  q_j = \left[ \begin{array}{c} \cos(\theta_j)\sin(\gamma_j)+\cos(\gamma_j)\cos(\phi_j)\sin(\theta_j) \\
                                \cos(\gamma_j)\cos(\phi_j)\cos(\theta_jg)-\sin(\gamma_j)\sin(\theta_j) \\
                               -\cos(\gamma_j)\sin(\phi_j)  \end{array} \right] 
$$
and $P_i = (x_i,y_i,z_i)$ into equations~(\ref{ullman3d}). The determinant of the Jacobean matrix can 
be computed explicitely. It is a polynomial in the $2$ unknown position variables $z_1,z_2$ and a trigonometric 
polynomial in the $6$ unknown camera orientation parameters $\theta_1,\phi_1,\gamma_1,\theta_2,\phi_2,\gamma_2$:
\begin{eqnarray*}
\det(J) &=& \sin^2(\phi_1) \\
       &\cdot& \sin^2(\phi_2) \\
       &\cdot& (A \cos(\theta_1)+B \sin(\theta_1)) \\
       &\cdot& (A \cos(\theta_2)+B \sin(\theta_2)) \\
       &\cdot& (\cos(\phi_2)\sin(\phi_1)(A\cos(\theta_1)+B\sin(\theta_1)) \\
         && +  \sin(\phi_2)(D \sin(\phi_1)\sin(\theta_1-\theta_2)+\cos(\phi_1)(C \cos(\theta_2)-B \sin(\theta_2))))  \; , 
\end{eqnarray*}
where 
$$ A = (y_2 z_1-y_1 z_2), B = (x_2 z_1-x_1 z_2), C = (y_1 z_2-y_2 z_1), D = (y_1 x_2-x_1 y_2) \; . $$
In general, this determinant is nonzero and by the implicit function theorem, the reconstruction 
is locally unique. \\

The main idea (due to Ullman) for the actual inversion is to first find vectors
$u_{ij}$ in the intersection lines of the three planes. For every pair $(i,j)$
of two cameras, the intersection line can be expressed in two ways:
$$ \alpha_{ij} p_i + \beta_{ij} q_i = \alpha_{ij} p_j + \beta_{ij} q_j $$
the projection of the two points produces equations
$$ \alpha_{ij} p_i \cdot P_k + \beta_{ij} q_i \cdot P_k = \alpha_{ij} p_j \cdot P_k + \beta_{ij} q_j \cdot P_k \; . $$
Because $a_{ik}  = p_i \cdot P_k,  b_{ik} = q_i \cdot P_k$ are known, these are $2$ equations
for each of the three pairs of cameras and each of the 4 unknowns
$\alpha_{ij},\beta_{ij},\gamma_{ij},\delta_{ij}$. Because additionally $\alpha_{ij}^2+\beta_{ij}^2=1$,
$\gamma_{ij}^2 + \delta_{ij}^2=1$, the values of $\alpha_{ij},\beta_{ij},\gamma_{ij},\delta_{ij}$ are
determined. \\

On page 194 in the book \cite{Ullman}, there are only 4 equations needed, not
5 as stated there to solve for the intersection lines of the planes. With 5 equations
the number of ambiguities is reduced. Actually, the Ullman equations with 4 equations 
have finitely many additional solutions which do not correspond to point-camera 
configurations. They can be detected by checking what projections they produce.  \\

We aim to find vectors $(\alpha_{ij},\beta_{ij})$ in the plane $i$ and coordinates 
$(\gamma_{ij},\delta_{ij})$ in the plane $j$
in the intersections of each pair $(i,j)$ of photographs. 
Taking the dot products with the two points $P_1,P_2$ gives the equations
\begin{eqnarray}
\label{ullmanequations}
 \alpha_{ij} u_{i1}  + \beta_{ij} v_{i1} &=& \gamma_{ij} u_{j1}  + \delta_{ij} v_{j1}  \\
 \alpha_{ij} u_{i2}  + \beta_{ij} v_{i2} &=& \gamma_{ij} u_{j2}  + \delta_{ij} v_{j2}  \\
 \alpha_{ij}^2 + \beta_{ij}^2 = 1 &,&  \gamma_{ij}^2 + \delta_{ij}^2 = 1 \; . 
\end{eqnarray}
They can be explicitely solved, evenso the formulas given by the computer algebra system 
are too complicated and contain hundreds of thousands of terms. 
Each of the above equations is of the form
$$ a x + b y = c u + d v,  e x + f y = g u + h v, x^2+y^2=u^2+v^2=1  \; . $$
Geometrically, it is the intersection of two three dimensional planes and two three dimensional cylinders in 
four dimensional space. From the first two equations, we have
$$  x = A u + B v,   y = F u + G v  \; . $$
By writing $u = \cos(t), v = \sin(t)$ the equation $x^2+y^2=1$ and replacing $\cos^2(t) = (\cos(2 t) + 1)/2$
$\sin^2(t) = (1-\cos(2 t))/2,  \sin(t) \cos(t) = \sin(2t)/2$, we get a quadratic equation for $\cos(2t)$
which has the solution 
$$ \cos(2 t) =  \frac{-(S T + W \sqrt{S^2-T^2+W^2})}{S^2+W^2} $$
with $U = (A^2+F^2)/2; V = (B^2+G^2)/2; W = (A B + F G); S = U-V; T = U+V-1$. We see that there are
8 solutions to the equations(\ref{ullmanequations}). Four of these solutions are
solutions for which 
$\alpha_{ij} p_i  + \beta_{ij} q_i - \alpha_{ij} p_j - \beta_{ij} q_j$  is perpendicular to the plane containing
the three points. These solutions do not solve the reconstruction problem and these branches of the algebraic
solution formulas are discarded. 
There are $4$ solutions to each Ullman equation which lead to solutions to the reconstruction problem. \\

Assume we know the three intersection lines in each plane. Because the ground camera plane is fixed, we know
two of the intersection lines. Let's denote by $U$ and $V$ the unit vectors in those lines. We have to find 
only the third intersection line which contains a unit vector $X$. 
This vector $X=(x,y,z)$ can be obtained by intersecting two cones. 
Mathematically, we have to solve the system $X \cdot U = r, X \cdot V = s, |X|=1$. 
This leads to elementary expressions by solving a quadratic equation. \\

Once we know the intersection lines, we can get the points
$P_1,P_2$ by finding the intersection of normals lines to the image points in the photographs. \\

The Ullman equations have $4$ solutions maximally. Because there are three
intersection lines we expect $4^3=64$ solutions in total in general. \\

If the normals to the cameras are coplanar, the problem reduces to a 
two-dimensional problem by turning the coordinate system so that the 
intersection line is the $z$-axes. This situation is what Ullman calls the 
{\bf degenerate case}. After finding the intersection line, we are 
directly reduced to the two-dimensional Ullman problem. 
\end{proof}

The fact that there are solutions to the Ullman equation
which do not lead to intersection lines of photographic planes could have been 
an additional reason for Ullman to add a 4'th point. Adding a 4'th point reduces the
number of solutions from 64 to 2 if the four points are noncoplanar but it makes 
most randomly chosen projection data unreconstructable. With three points, there is 
an open and algebraically defined set for which a reconstruction is not possible and
and open algebraically defined set on which the reconstruction is possible and 
locally unique. The boundary of these two sets is the image of the set 
${\rm det}(F)=0$. 

\begin{center}
\parbox{16.2cm}{\scalebox{1.20}{\includegraphics{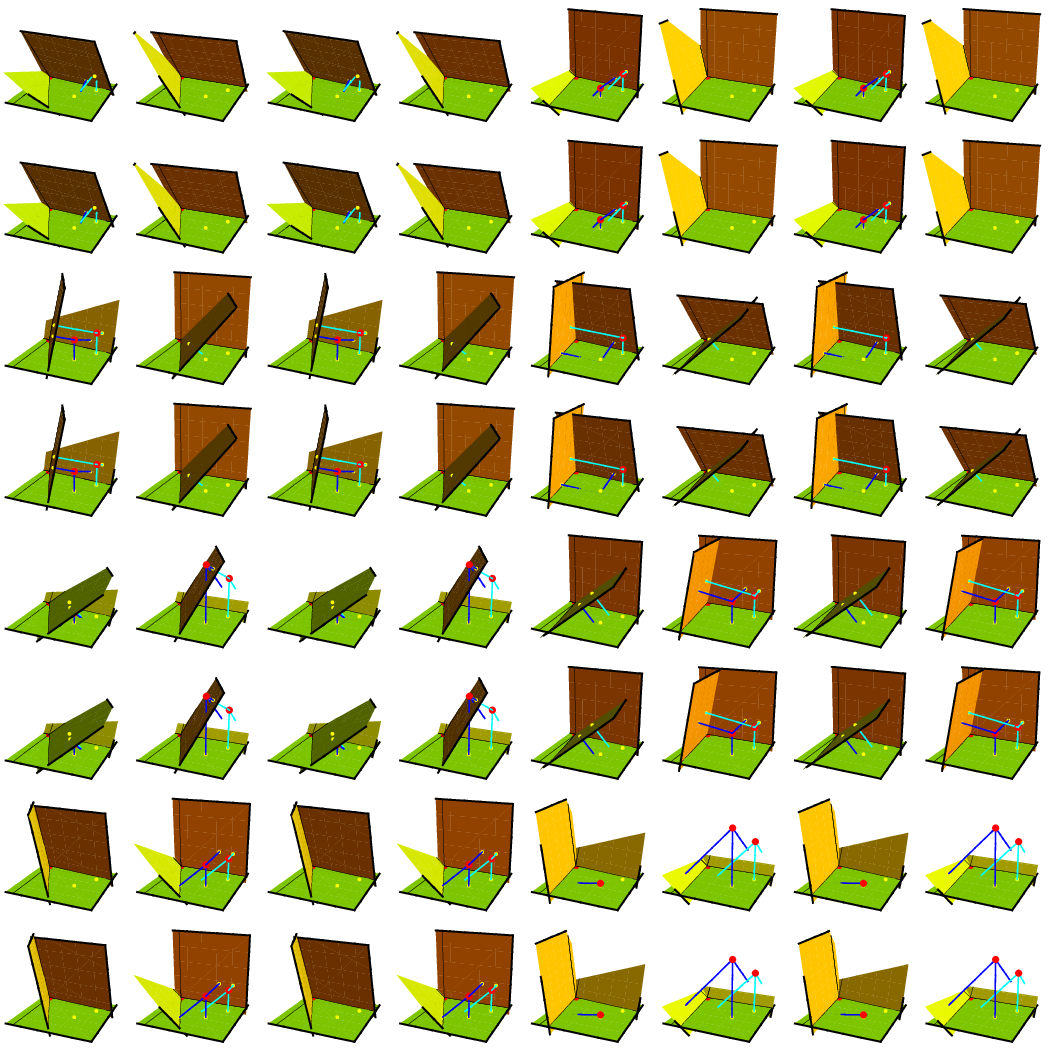}}}
\end{center}
\begin{fig}
64 solutions to the reconstruction problem in a particular case.
\end{fig}

\section{When is the reconstruction possible?}

Given three photographs each showing three points. As usual, we know which points correspond.
How can we decide whether there is a point-camera configuration which realizes this picture? 
Of course, we have explicit formulas, they do not illustrate the geometry very well. \\

Define for two complex numbers $A,B$ the interval $I(A,B)$ of possible angles
$$    {\rm arg} (\frac{ e^{i \theta} - A  }{ e^{i \theta} - B } ) \; , $$
where $\theta \in [0,2\pi)$.

\begin{center}
\parbox{6.2cm}{\scalebox{0.50}{\includegraphics{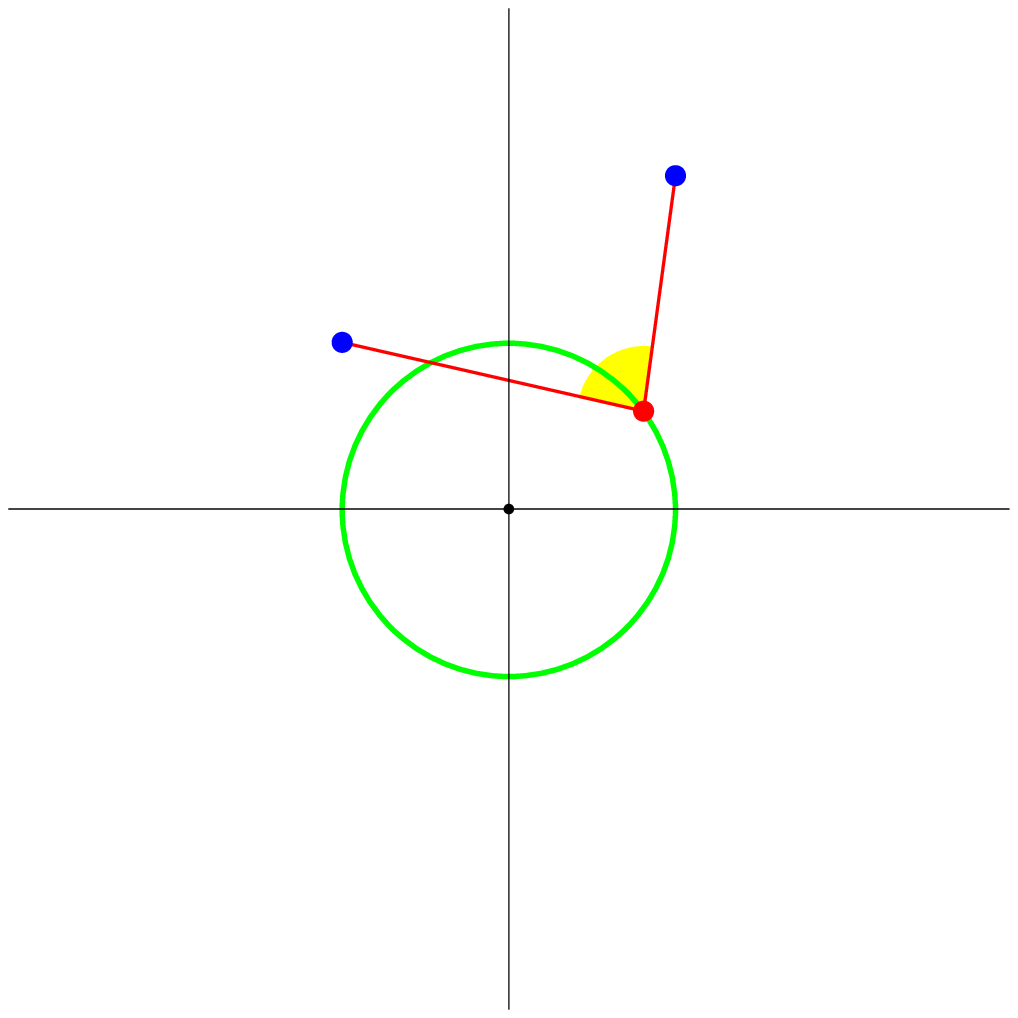}}}
\end{center}
\begin{fig}
The range of angles $I(A,B)$. 
\end{fig}

The following lemma deals with the equations which determine the intersection 
lines of the camera planes. 

\begin{lemma}
The equations
$$ a x + b y = c u + d v,  e x + f y = g u + h v, x^2+y^2=u^2+v^2=1 $$
can be solved for the unknown $x,y,u,v$ for any values of $a,b,c,d,e,f,g,h$
for which
$$ {\rm arg}(\frac{c+id}{g+ih}) \in I(\frac{c+id}{a+ib},\frac{g+ih}{e+if}) $$
\end{lemma}
\begin{proof}
Define $p=a+ib, q=c+id, r=e+i f, s=g+i h$.
We look for two complex numbers $z=x-iy,w=u-iv$ of modulus $1$ such that
${\rm Re}(z p) = {\rm Re}( w q), {\rm Re}(u r) = {\rm Re}( v s)$. Therefore
${\rm arg}(z p - w q) = \pi/2, {\rm arg}(z r - w s) = \pi/2$. With $z=e^{i \theta},w=e^{i \phi}$,
this defines two curves on the torus. The solutions are the intersection points.
If ${\rm arg}(q/s) \in I(q/p,s/r)$, there is a solution to the problem.
\end{proof}

\section{Final remarks}

{\bf Explicit implementations}. \\
We have implemented the reconstruction explicitely in Mathematica 6, a computer algebra
system in which it is now 
possible to manipulate graphics parameters. We have programs, which invert the 
nonlinear equations on the spot, both in two and three dimensions.  \\

\begin{center}
\parbox{16.8cm}{
\parbox{8.2cm}{\scalebox{0.40}{\includegraphics{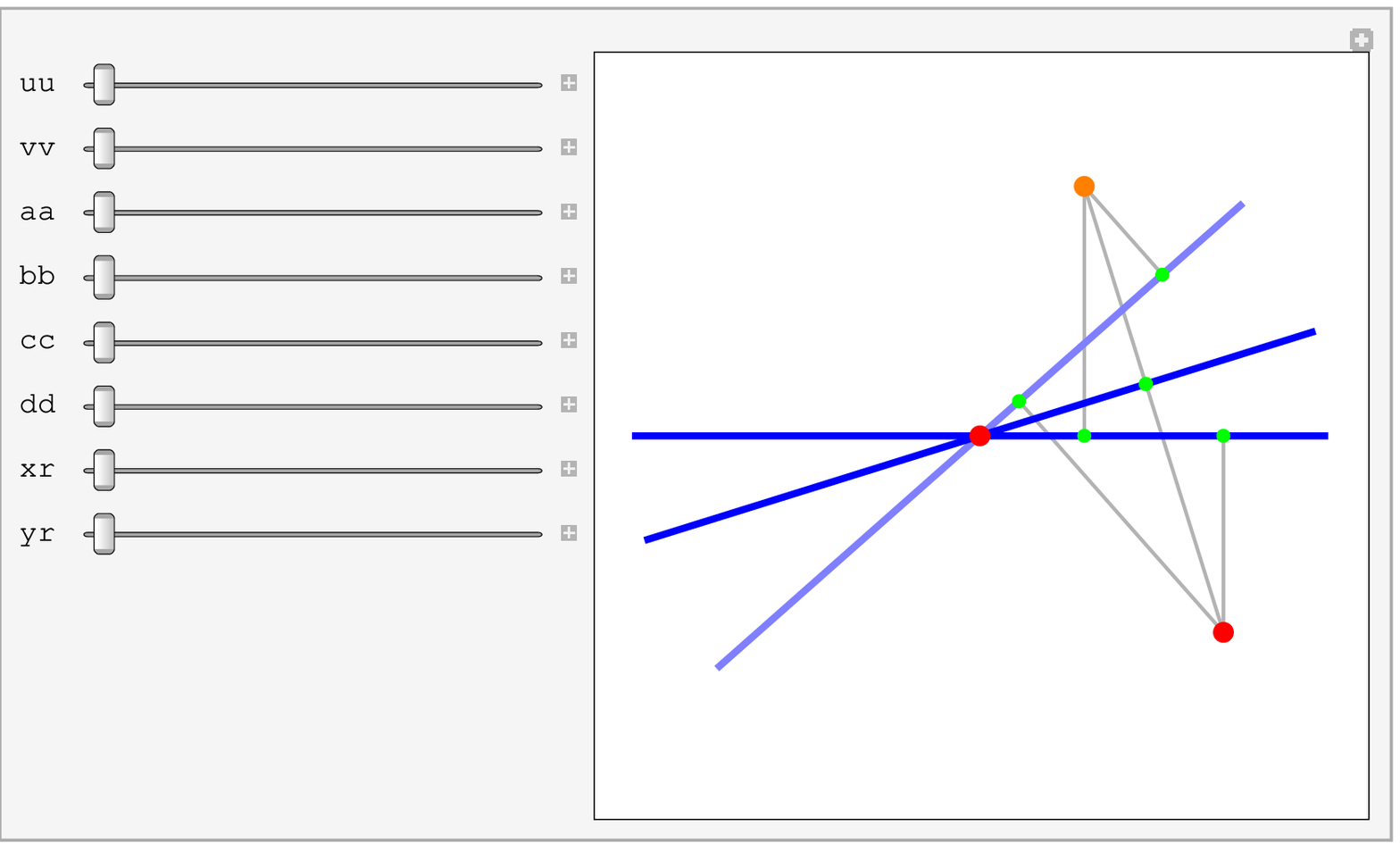}}}
\parbox{8.2cm}{\scalebox{0.40}{\includegraphics{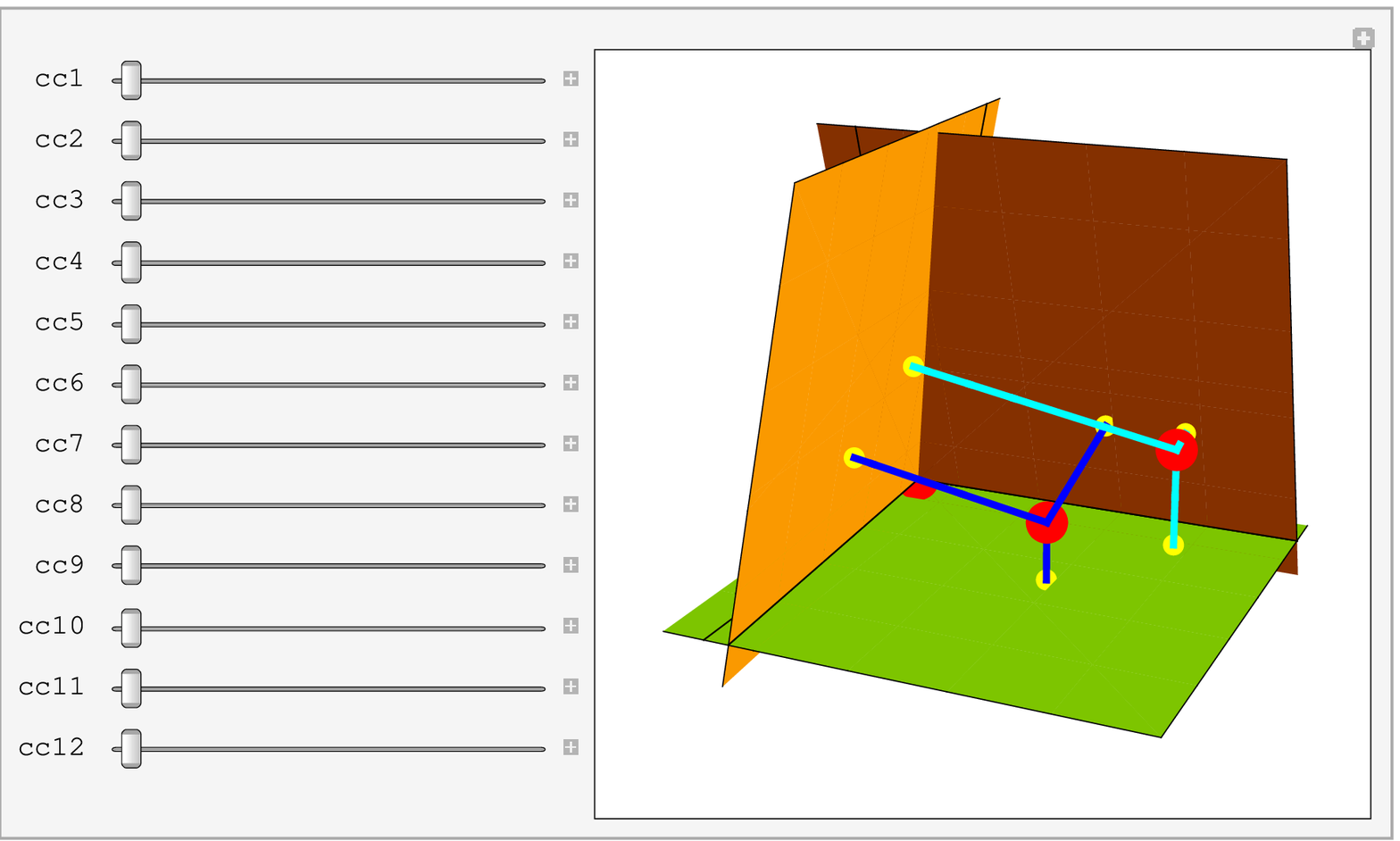}}}
}
\end{center}
\begin{fig}
Interactive demonstration of the reconstruction in two 
and three dimensions with Mathematica. The user can change each of the 
image parameters and the computer reconstructs the cameras and the points. 
We will have this programs available on the Wolfram demonstration project. 
\end{fig}

{\bf Higher dimensions}. \\
How many points are needed in $d$ dimensions for $3$ orthographic cameras to locally have a unique
reconstruction? In $d$ dimensions, an orthographic camera has $f=d(d-1)/2 + (d-1)$ parameters
and the global Euclidean symmetry group has dimension $g = d + d (d-1)/2$. The dimension relations are
\begin{eqnarray*}
 n d + m f &=& (d-1) n m + g  \\
 f &=& d(d-1)/2 + (d-1)   \\
 g &=& d + d (d-1)/2  \; . 
\end{eqnarray*}

This gives \\

\begin{tabular}{|l|llll|} \hline
dimension & $n(m)$                 &  $n(2)$ & $n(3)$  &  $n(4)$ \\ \hline
dim=2:  &   $n = (2m-3)/(m-2)$     &    -    &  $3$    &  $3$ \\ \hline
dim=3:  &   $n = (5m-6)/(2m-3)$    &   $3$   &  $3$    &  $3$ \\ \hline
dim=4:  &   $n = (9m-10)/(3m-4)$   &   $4$   &  $4$    &  $4$ \\ \hline 
\end{tabular}

\vspace{1cm}

%     d=4; f= d(d-1)/2 + (d-1); g= d+d (d-1)/2; Solve[n d + m f  == (d-1) n m + g,{n}]

In any dimension, there is always a reflection ambiguity.  \\

{\bf Other cameras}. \\
The structure from motion problem can be considered for many other camera types. The most common is the 
pinhole camera, a perspective camera. In that case, two views and 7 points are enough to determine 
structure from motion locally uniquely, if the focal parameter is kept the same in both shots and
needs to be determined too.
We have studied the structure from motion problem for spherical cameras in detail in the paper
\cite{KnillRamirezOmni} and shown for example that for 
three cameras and three points in the plane a {\bf unique} reconstruction is possible if both the camera and point 
sets are not collinear and the 6 points are not in the union of two lines. 
This uniqueness result can be proven purely geometrically using Desarques theorem and is sharp: weakening any 
of the three premises produces ambiguities, where the two line ambiguity was the hardest to find. \\

{\bf Other fields}. \\
The affine structure of motion problem can be formulated over other fields too, and not 
only over the field of reals $k={\bf R}$ or complex numbers $k={\bf C}$. 
The space $S$ is a $d$- dimensional vector space over some field $k$. 
A camera is a map $Q$ from $N$ to a $(d-1)$-dimensional linear
subspace satisfying $Q^2=Q$. A point configuration $\{ P_1,P_2,...,P_n \; \}$
and a camera configuration $\{Q_1, \dots, Q_m \; \}$ 
define image data $Q_j(P_i)$. The task is to reconstruct
from these data the points $P_i$ and the cameras $Q_j$. 
If the field $k$ is finite, the structure from motion problem is a problem in a finite
affine geometry. If the inversion formulas derived over the reals make sense in that field, then they 
produce solutions to the problem. "Making sense" depends for example, whether we can take square 
roots. We might ask the field $k$ to be algebraicalliy complete so that 
a reconstruction is possible for all image data.  \\

{\bf A question}. For orthographic cameras in the plane, the only ambiguities are a reflection.
One can extend the global symmetry group $G$ so that the map $F$ becomes injective. Can one extend
the group in three dimensions also to make the structure from motion map $F$ globally injective? 
To answer this, we would need to understand better the structure of the finite set 
$F^{-1}(a)$ if $a$ is in the image of $F$.

\vspace{12pt}
\bibliographystyle{plain}
\bibliography{3dimage}
\end{document}